\pdfoutput=1

\documentclass[11pt,a4paper]{article}
\usepackage[hyperref]{acl2023}
\usepackage{times}
\usepackage{inconsolata}
\usepackage{latexsym}
\usepackage{danudefs}
\usepackage[algoruled,vlined,nofillcomment,linesnumbered]{algorithm2e}
\usepackage{color}
\usepackage{booktabs}
\usepackage{xspace}
\usepackage{url}
\usepackage{caption}
\usepackage{subcaption}
\usepackage{graphicx} 
\usepackage{adjustbox}
\graphicspath{{figures/}}

\usepackage[english]{babel}
\usepackage{hyperref}
\addto\captionsenglish{%
}
\addto\extrasenglish{%
}

\usepackage{acronym}
\acrodef{MLM}[MLM]{Masked Language Model}
\acrodef{SoTA}[SoTA]{state-of-the-art}

\newtheorem{thm}{Theorem}
\newtheorem{cor}{Corollary}
\newtheorem{rem}{Remark}

\newtheorem{defn}{Definition}

\newcommand{\Pro}{\mathbb{P}}

\title{A Neighbourhood-Aware Differential Privacy Mechanism\\ for Static Word Embeddings}

    \author{Danushka Bollegala$^{1,2}$\thanks{Contact author: {\tt danushka@liverpool.ac.uk}} \And Shuichi Otake$^3$ 
    \\  ~~~~~~~~~~~~~~~~~~~~~~~~~~~~~~~~~~~~University of Liverpool$^1$, Amazon$^2$, National Institute of Informatics$^3$\\
    ~~~~~~~~~~~~~~~~~~~~~~~~~~~~~~~~~~~International Professional University of Technology in Tokyo$^4$
          \And Tomoya Machide$^4$ \And  Ken-ichi Kawarabayashi$^3$}

\date{}

\begin{document}
\maketitle

\begin{abstract}
We propose a Neighbourhood-Aware Differential Privacy (NADP) mechanism considering the neighbourhood of a word in a pretrained static word embedding space to determine the minimal amount of noise required to guarantee a specified privacy level.
We first construct a nearest neighbour graph over the words using their embeddings, and factorise it into a set of connected components (i.e. neighbourhoods).
We then separately apply different levels of Gaussian noise to the words in each neighbourhood, determined by the set of words in that neighbourhood.
Experiments show that our proposed NADP mechanism consistently outperforms multiple previously proposed DP mechanisms such as Laplacian, Gaussian, and Mahalanobis in multiple downstream tasks, while guaranteeing higher levels of privacy.
\end{abstract}

\section{Introduction}
\label{sec:intro}

Increasingly more NLP models have been trained on private data such as medical conversations, social media posts and personal emails~\cite{Abdalla:2020,Lyu:2020,Song:2019}.
However, we must ensure that sensitive information related to user privacy is not leaked during any stage of the model training process.
To protect user privacy, Differential Privacy (DP) mechanisms  add random noise to the training data~\cite{feyisetan-kasiviswanathan-2021-private,krishna-etal-2021-adept,Feyisetan:2020}.
However, it remains a challenging task to balance the trade-off between user \textbf{privacy} vs. \textbf{performance} in downstream NLP tasks.

We propose \textbf{Neighbourhood-Aware Differential Privacy} (NADP) mechanism, which consists of three steps.
First, given a set of words, we compute a nearest neighbour  graph considering the similarity between the words (represented by the vertices of the nearest neighbour graph) computed using their word embeddings.
Second, we compute the connected components in the nearest neighbour graph to find the \emph{neighbourhoods} of words.
Third, we apply Gaussian noise to all words in each neighbourhood, such that the variance of the noise is determined by the words in that neighbourhood.

\begin{figure}[t]
    \centering
    \includegraphics[height=40mm]{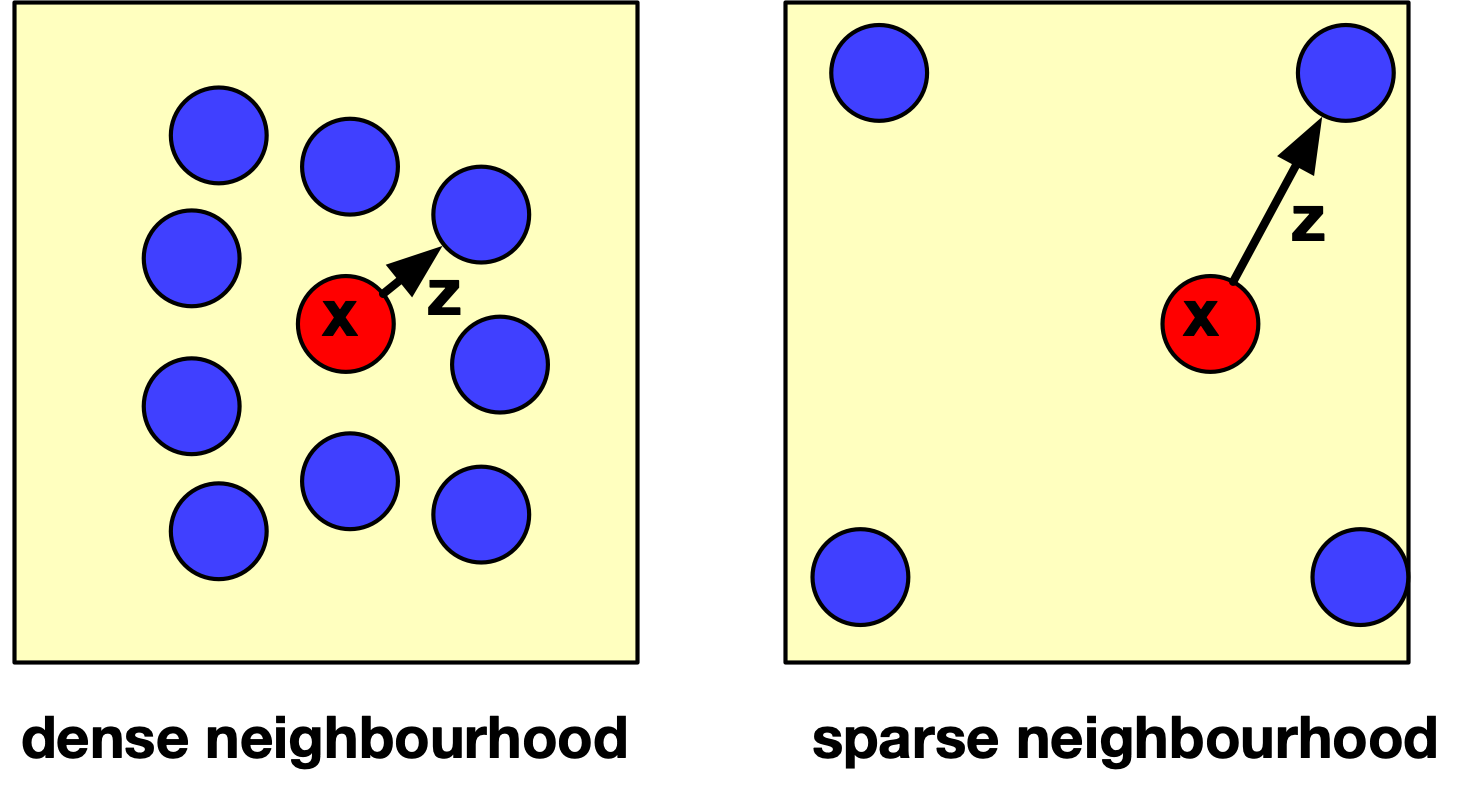}
    \caption{Anonymizing a target word (shown in red) in a dense (left) vs. a sparse (right) neighbourhoods of words (shown in blue). In the sparse neighbourhood, NADP adds a higher level of perturbation noise $\vec{z}$ to the target word embedding $\vec{x}$ in order to protect its privacy by disguising it among its neighbours, while in a dense neighbourhood it adds less noise.}
    \label{fig:idea}
\end{figure}

As illustrated in \autoref{fig:idea}, if all words in a neighbourhood are highly similar to each other (i.e. a \emph{dense} neighbourhood), it would require less perturbation noise to anonymise a word because the addition of small noise can easily \emph{hide} the corresponding word embedding among its neighbours.
On the other hand, if the words in a neighbourhood are not very similar to each other (i.e. a \emph{sparse} neighbourhood) we must add higher levels of perturbation noise to a word embedding because its nearest neighbour would be further away in the embedding space.
Because words in a language is a discrete set (unlike images for example), there does not exist a word corresponding to all points in the embedding space.
Therefore, if we do not add sufficient amount of noise in a sparse neighbourhood, we run the risk of easily discovering the target word via a simple nearest neighbour search.
Instead of adding the same level of noise to all words in a vocabulary as done in prior DP mechanisms, NADP attempts to minimise the total amount of noise by assigning low noise in dense neighbourhoods and high noise in sparse neighbourhoods.
NADP has provable DP guarantees as shown by our theoretical analysis.
Moreover, NADP has the following desirable properties that makes it attractive when used for NLP tasks.

\paragraph{(a)} In NADP, \textbf{noise vectors are sampled from the Gaussian distribution}.
Many static word embedding algorithms \cite{Pennington:EMNLP:2014,Arora:TACL:2016,Milkov:2013} learn embeddings in the $\ell_{2}$ space.
Moreover, the squared $\ell_{2}$ norm of a word embedding is known to positively correlate with the frequency of the word in the training corpus~\cite{Arora:TACL:2016}, while the joint co-occurrence probability of a set of words positively correlates with the squared $\ell_{2}$ norm of the sum of the corresponding word embeddings~\cite{Bollegala:AAAI:2018}.
Therefore, it is natural to consider Gaussian noise, which corresponds to the $\ell_{2}$ embedding space used by many static word embedding learning methods rather than the more widely-used Laplacian noise, which relates to the $\ell_{1}$ norm.

\paragraph{(b)}Unlike previously proposed DP mechanisms for word embeddings~\cite{Feyisetan:2020,feyisetan-kasiviswanathan-2021-private,krishna-etal-2021-adept,xu-etal-2020-differentially}, NADP \textbf{dynamically adjusts the level of noise added to a word embedding considering its neighbourhood}.
This enables us to optimally allocate a fixed noise budget over a vocabulary.

\paragraph{(c)} NADP \textbf{adds noise directly to the word embeddings} and does not perform decoding after the noise addition step~\cite{krishna-etal-2021-adept}.
Decoding is a deterministic process and does not affect DP. 
Many NLP applications such as text classification, clustering etc. require the input text to be represented in some vector space, and we can use the noise-added input text representations straightaway in such applications without requiring to first decode it back to text.
In situations where users train word embeddings on private data on their own and only send/release the embeddings to external machine learning services, we only need to anonymise the word embeddings~\cite{feyisetan-kasiviswanathan-2021-private}.

\noindent\textbf{Results:} \emph{Utility experiments} (\autoref{sec:exp-utility}) conducted over four downstream NLP tasks show that NADP consistently outperforms previously proposed Laplacian, Gaussian and Mahalanobis mechanisms in downstream tasks.
We conduct \emph{privacy experiments} (\autoref{sec:exp-privacy})  to evaluate the level of privacy guaranteed by a DP mechanism for word embeddings. 
Specifically, we estimate the probability of correctly predicting a word from its perturbed word embedding using the overlap between nearest neighbour sets.
To evaluate the level of privacy protected for the entire set of word embeddings, we compute the skewness of the distribution of prediction probabilities.
We find that NADP reports near-zero skewness values across a broad range of privacy levels, $\epsilon$, which indicates significantly stronger privacy guarantees compared to other DP mechanisms.
Source code implementation of our NADP is publicly available.\footnote{\url{https://github.com/shuichiotake/NADP}}

\section{Related Work}
\label{sec:related}

Learning models from data with DP guarantees has been studied under \emph{private learning}~\cite{Kasiviswanathan:2008}.
\newcite{DPSGD}~proposed a DP stochastic gradient descent by adding Gaussian noise to the gradient of the loss function.
\newcite{Rogers:2016uh} combined multiple DP algorithms using adaptive parameters.
However, compared to continuous input spaces such as in computer vision~\cite{Zhu:2020}, DP mechanisms for the discrete inputs such as text remain understudied.

\newcite{wang-etal-2021-certified} proposed \emph{WordDP} to achieve certified robustness against word substitution attacks in text classification.
However, WordDP does \emph{not} seek DP protection for the training data as we consider here, and uses DP randomness for certified robustness during inference time with respect to a testing input.
\newcite{krishna-etal-2021-adept} proposed AdePT, an autoencoder-based approach to generate differentially private text transformations.
However, \newcite{habernal-2021-differential} showed that AdePT is \emph{not} differentially private as claimed and proved weaker privacy guarantees. 
DPText~\cite{Alnasser:2021,Beigi:2019} uses an autoencoder to obtain a text representation and adds Laplacian noise to create private representations.
However, \newcite{habernal-2022-reparametrization} proved that the use of reparametrisation trick for the inverse continuous density function in DPText is inaccurate and that DPText violates the DP guarantees.
Such prior attempts show the difficulty in developing theoretically correct DP mechanisms for NLP.

\newcite{lyu-etal-2020-differentially} proposed DP Neural Representation (DPNR) to preserve the privacy of text representations by first randomly masking words from the input texts and then adding Laplacian noise.
However, unlike NADP DPNR uses a neighbourhood insensitive fixed Laplacian noise distribution. 
\newcite{Feyisetan:2020} proposed a DP mechanism where they first add Laplacian noise to word embeddings and then return the nearest neighbour to the noise-added embedding as the output. 
However, the $\ell_{2}$ norm of the noise vector scales almost linearly with the dimensionality of the embedding space.
To address this issue, in their subsequent work~\cite{feyisetan-kasiviswanathan-2021-private}, they projected the word embeddings to a lower-dimensional space before adding Laplacian noise.

\newcite{xu-etal-2020-differentially} proposed Mahalanobis DP mechanism, which adds elliptical noise considering the covariance structure in the embedding space. 
Unlike the Gaussian or Laplacian mechanisms, Mahalanobis mechanism adds heterogeneous noise along different directions such that words in sparse regions in the embedding space have sufficient likelihood of replacement without sacrificing the overall utility.
They show that Mahalanobis mechanism to be superior to Laplacian mechanism.
Mahalanobis mechanism is a special instance of metric (Lipschitz) DP originated in privacy-preserving geolocation studies~\cite{Andres:2013}, where Euclidean distance was used as the distance metric.
Although metric DP considers the distance between two data points, it does not consider all of the nearest neighbours for each data point when deciding the level of noise that must be applied to a particular data point, unlike our NADP mechanism.
\newcite{li-etal-2018-towards} used adversarial learning to build NLP models such as part-of-speech (PoS) taggers that cannot predict the writer's age or sex, while can accurately predict the PoS tags.
Despite their empirical success, this approach does not have any formal DP guarantees.
In contrast, our focus is provably DP mechanisms with formal guarantees.

All of the prior work described thus far, except DPNR and AdePT, focus on static word embeddings as we do in this paper.
A natural future extension of this work is DP mechanisms for the contextualised embeddings.
However, computational and practical properties of static word embeddings such as,  being lightweight to both compute and store, are attractive for resource (e.g. GPU and RAM) limited mobile devices.
Considering that such personal mobile devices are used by billions of users and contain highly private data, DP mechanisms for static word embeddings remains an important research topic. 
Moreover, \newcite{gupta-jaggi-2021-obtaining} showed that it is possible to  distil static word embeddings from pretrained language models that have comparable performance to contextualised word embeddings.

\section{DP for Word Embeddings}
\label{sec:DP}

Let us denote the $d$-dimensional embedding of a word $x$ in a vocabulary $\cX$ by a vector $\vec{x} \in \R^{d}$.
We can consider a word embedding algorithm as a function $f : \mathbb{X} \rightarrow \R^d$ that maps the words in a discrete vocabulary space $\mathbb{X}$ to a $d$-dimensional continuous space $\R^d$.
We can use a distance metric, $\Gamma$, defined in the embedding space to measure the distance $\Gamma(\vec{x}_i, \vec{x}_j)$ between two words $x_i$ and $x_j$ such as the Euclidean distance.
We can then find the set of top-$m$ nearest neighbours, $\cS_{m}(x)$, from $\cX$ using $\Gamma$ such that for any $y \in \cS_{m}(x)$ and $y' \notin \cS_m(x)$, $\Gamma(x, y) \leq \Gamma(x, y')$ holds.
The Jaccard similarity, $\mathrm{Jaccard}(x,y)$, between two words $x$ and $y$ is defined using their neighbourhoods as in \eqref{eq:jaccard}.
\begin{align}
    \label{eq:jaccard}
    \mathrm{Jaccard}(x,y) = \frac{|\cS_m(x) \cap \cS_m(y)|}{|\cS_m(x) \cup \cS_m(y)|}
\end{align}
We define two words $x, y \in \cX$ to be in a symmetric \emph{neighbouring} relation, $x \simeq y$, if the following two conditions are jointly satisfied:
\begin{description}
\item[(a)] $x \in \cS_{m}(y)$ or $y \in \cS_{m}(x)$, and
\item[(b)] $\mathrm{Jaccard}(x,y) \geq \tau$ for a given threshold $\tau \in [0,1]$.
\end{description}
One could use conjunction instead of disjuction in condition (a) to enforce a mutual nearest neighbour relation.
However, doing so results in a large number of small isolated neighbourhoods because two words might not be \emph{mutual} nearest neighbours unless they are synonyms (or highly related).
Relaxing the condition (a) to a disjunction would form neighbourhoods where one word might be a neighbour of another but not the inverse such as in hypernym-hyponym pairs.
For example, \emph{colour} could be a top nearest neighbour of \emph{crimson}, but \emph{crimson} might not be a top nearest neighbour of \emph{colour}, because there are other prototypical colours such as \emph{red}, \emph{green}, \emph{blue}, etc. than \emph{crimson}.

Let us formally define DP for word embeddings.
Because each word is assigned a vector by the word embedding learning algorithm, we can add noise to the embedding vectors to \emph{disguise} a word among its nearest neighbours in the embedding space.
However, in doing so we will be perturbing the semantics in the embeddings, thus potentially hurting downstream task performance.
Therefore, there exists a \emph{trade-off} between the amount of privacy that can be guaranteed by adding random noise to the embeddings vs. the performance of a downstream NLP task that use those embeddings.
A random mechanism operating on word embeddings is said to be DP if Definition~\ref{def:DP} holds.

\begin{defn}[Differential Privacy]
\label{def:DP}
A random mechanism $M$ that takes in a vector in the embedding space $\mathbb{X}$ and maps into a space $\mathbb{Y}$ (i.e. $M : \mathbb{X} \rightarrow \mathbb{Y}$) is $(\epsilon, \delta)$-DP with $\epsilon \geq 0$ and $\delta \in [0,1]$,
 if for every pair of neighbouring inputs $\vec{x}, \vec{x}' \in \cX$ and  every possible measurable output set  $\cT \in \cY$ the relationship given by \eqref{eq:DP} holds:
\begin{align}
\label{eq:DP}
\mathrm{Pr}[M(\vec{x}) \in \cT] \leq \exp(\epsilon) \mathrm{Pr}[M(\vec{x}') \in \cT] + \delta
\end{align}
\end{defn}
Here, $\epsilon$ represents the level of privacy ensured by $M$ and smaller $\epsilon$ values result in stronger privacy guarantees.
The global $\ell_{2}$ sensitivity of the embedding space is defined as $\Delta = \sup_{x,x' \in \cX,x \simeq x'} \norm{\vec{x} - \vec{x}'}$.
Given a set of word embeddings, $\Delta$ can be estimated empirically by computing the maximum Euclidean distance between a word $x$ and its most distant neighbour $x'$ in $\cS_m(x)$. 
As an extreme case, let us consider the smallest possible neighbourhood size corresponding to $m = 2$.
Estimating $\Delta$ in this case would amount to finding the maximum Euclidean distance between any pair of neighboring words $x, x' \in \cV$.
Moreover, the $\Delta$ estimated for $m = 2$ will be larger than the $\Delta$ estimated for any other $m (>2)$ neighbourhood sizes.
Therefore, $\Delta$ is independent of $m$ and can be estimated via a deterministic process (i.e. measuring all pairwise Euclidean distances) from a given set of word embeddings.

\subsection{Gaussian Mechanism}
\label{sec:gaussian}

Gaussian mechanism uses $\ell_{2}$ norm for estimating the sensitivity due to perturbation and is a more natural fit for word embeddings than, for example, the Laplace mechanism, which is associated with the $\ell_{1}$ norm.
Therefore, we use the Gaussian mechanism as the basis for our proposal.

Let us consider a multivariate zero-mean isotropic Gaussian noise distribution, 
$\cN(\vec{0}, \sigma^2 \mat{I}_{d})$, where $\mat{I}_{d}$ is the unit matrix in the $d$-dimensional real space and 
$\sigma$ is the standard deviation.
For each word, $x \in \cX$, we sample a random vector $\vec{z} \sim \cN(\vec{0}, \sigma \mat{I}_{d})$ and create a noise-added embedding $M_{g}(\vec{x})$ for $x$ as given by \eqref{eq:gauss-noise}.
\begin{align}
\label{eq:gauss-noise}
M_{g}(\vec{x}) = \vec{x} + \vec{z}
\end{align}
This Gaussian mechanism uses the same $\sigma$ for all words in the vocabulary and is $(\epsilon,\delta)$-DP as claimed in in Theorem~\ref{th:gauss}.\footnote{As a direct extension, we could set a different standard deviations for each dimension of the embedding space. 
However, doing so did not result in significant performance gains in our preliminary investigations despite the increased parameters.}

\begin{thm}
\label{th:gauss}
For any $\epsilon, \delta \in (0,1)$, the Gaussian mechanism with $\sigma = \Delta\sqrt{2\log(1.25/\delta)}/\epsilon$ is $(\epsilon,\delta)$-DP.
\end{thm}
The proof of Theorem~\ref{th:gauss} can be found in the Appendix~A in~\cite{Dwork:2014}.

\section{Neighbourhood-Aware Differential Privacy (NADP)}
\label{sec:NADP}

\begin{algorithm}[t]
\DontPrintSemicolon
\caption{Nearest Neighbour Graph Construction}\label{algo:NN}
\KwSty{Inputs:} Word embeddings $\cX = \{\vec{x}_1, \ldots, \vec{x}_n\}$, top-$m$ for selecting neighbours, and similarity threshold $\tau \in [0,1]$. \;
\KwSty{Outputs:} Nearest neighbour graph $G(\cV,\cE)$\;
Initialise $\cV = \{x_1, \ldots, x_n\}$, $\cE = \{ \}$\;
\For{$i = 1, \ldots, n$}{
    \For{$x_j \in \cS_m(x_i)$}{
        \If{$\mathrm{Jaccard}(x_i, x_j) \geq \tau$}{
            $\cE = \cE + \{(i,j)\}$\;
        }
    }
}
Return $G(\cV,\cE)$
\end{algorithm}

Our proposed DP-mechanism, Neighbourhood-Aware Differential Privacy (NADP), consists of three main steps.
First, we create a nearest neighbour graph where vertices represent the words as described in \autoref{sec:NN}.
Next, we factorise this nearest neighbour graph into a set of mutually exclusive neighbourhoods by finding its connected components as described in \autoref{sec:CC}.
Finally, for the words that belong to each connected component, we add random noise sampled from Gaussian distributions with zero mean and \emph{different} standard deviations, determined according to the neighbourhood associated with the corresponding connected component. We prove that the proposed NADP mechanism is DP in \autoref{sec:noise}.

\subsection{Nearest Neighbour Graph Construction}
\label{sec:NN}

To represent the nearest neighbours of a set $\cX$ of words, we construct a nearest neighbour graph, $G=G(\cX, \simeq)$ with the symmetric neighbouring relation $\simeq$, vertex set $\cV(G) = \cX$ and edge set $\cE(G) = \{ (x, x^{\prime}) \mid x \simeq x^{\prime} \}$. 
Given the one-to-one mapping between words and the vertices in the graph, for notational simplicity we denote the $i$-th vertex of the graph by $x_i (\in \cX)$.
Two vertices $x_i$ and $x_j$ are connected by an edge $e_{ij} (\in \cE)$, if $x_i \simeq x_j$ holds between the corresponding words $x_i$ and $x_j$.
As already explained in \autoref{sec:DP}, we define two words $x_i, x_j \in \cX$ to be in a symmetric neighbouring relation, $x_i \simeq x_j$ if the following two conditions are jointly satisfied:
(a) $x_i \in \cS_{m}(x_j)$ or $x_j \in \cS_{m}(x_i)$, and
(b) $\mathrm{Jaccard}(x_i,x_j) \geq \tau$ for a predefined similarity threshold $\tau \in [0,1]$.

The pseudo code for constructing  the nearest neighbour graph is shown in \autoref{algo:NN}.
In our experiments, we set $m = 2$, which considers only the top-2 neighbours (i.e. $\cS_2$) to ensure only the highly similar neighbours are connected by edges in the nearest neighbour graph.
$\tau$ can be used to remove neighbours that have less similarity to a target word across the graph.
For example, by setting $\tau = 0.8$, we can ensure that no two words with neighbourhood similarity (measured using the Jaccard coefficient) less than $0.8$ will be connected by an edge in $\cG$.
We empirically study the effect of varying $\tau$ on NADP later in our experiments.

\subsection{Finding Connected Components}
\label{sec:CC}

Once a nearest neighbour graph $\cG$ is constructed for $\cX$, next we identify the regions of neighbours, which we refer to as the \emph{neighbourhoods}.
To consider tightly connected neighbourhoods, we propose to factorise $\cG$ into a set of mutually exclusive connected components following the procedure described in \autoref{algo:CC}.
We start by randomly selecting a word $x$ from $\cX$ and creating a neighbourhood $\cX_1$ consisting all of $x$'s neighbours.
We then remove the words in $\cX_1$ from $\cX$, and repeat this process until all words in $\cX$ are included in some neighbourhood.
The procedure described in \autoref{algo:CC} for obtaining connected components from $\cG$ is simple, efficient and obtains good DP performance in our experiments. 
Moreover, it does \emph{not} require the number of neighbourhoods, $k$, to be specified in advance as it would be the case for many clustering-based approaches for graph partitioning such as spectral clustering~\cite{Luxburg:2007}.
There is a possibility of obtaining long chains when computing connected components using \autoref{algo:CC}.
However, we did not encounter this issue in our experiments.
This is because the nearest neighbour relation that is defined in \autoref{sec:DP} requires both mutual nearest neighbourhood and high Jaccard similarity to be satisfied, which reduces the likelihood of forming long chains.
Exploring alternative methods for factorising a given graph into a set of mutually exclusive connected components is deferred to future work.

\begin{algorithm}[t]
\DontPrintSemicolon
\caption{Finding Connected Components}\label{algo:CC}
\KwSty{Inputs:} Nearest neighbour graph $G(\cV,\cE)$\; 
\KwSty{Outputs:} Connected components  $\{\cX_1, \ldots, \cX_k\}$\;
Define $k = 0$\;
Define $\cH_k = \{ \}$\;
\While{$\cX \setminus \cH_k \neq \emptyset$}{
    Choose $x \in \cX \setminus \cH_k$\;
    $k = k + 1$\;
    $\cX_k = \{x\}$\;
    Define $\cX' = \{x'|x' \simeq x\} \setminus \cX_k$\;
    $\cX_k = \cX_k \cup \cX'$\;
    $\cH_k = \cH_k \cup \cX_k$\;
    \While{$\cX' \neq \emptyset$}{
        $\cX' = \{x'|x' \simeq x \text{ for some } x \in \cX_k\} \setminus \cX_k$\;
        $\cX_k = \cX_k \cup \cX'$\;
        $\cH_k = \cH_k \cup \cX_k$\;
    }
}
Return $\{\cX_1, \ldots, \cX_k\}$
\end{algorithm}

\subsection{Perturbation of Word Embeddings}
\label{sec:noise}

In this section, we will first prove that NADP satisfies the DP conditions, and then present an algorithm that can be used to add perturbation noise to the words in each neighbourhood.
First note that the trivial relation $x \simeq x$ implies the set $\{ \norm{\vec{x}-\vec{y}} \mid x\simeq y \}$ is nonempty and hence we can consider the global $L_2$ sensitivity, $\Delta = \sup_{x \simeq y} \norm{\vec{x}-\vec{y}}$, for any two neighbouring words $x$ and $y$ in the given set of words $\cX$.
\newcite{pmlr-v80-balle18a} proved Theorem~\autoref{th:BW} that shows a set of word embeddings can be made differentially private by adding Gaussian noise sampled according to $\Delta$, where $\Phi(t)$ the Cumulative Density Function (CDF) of the standard univariate Gaussian distribution, given by \eqref{eq:CDF}.
\par\nobreak
{\small
\begin{align}
\label{eq:CDF}
    \Phi(t) = \frac{1}{\sqrt{2\pi}} \int_{-\infty}^{t} e^{-y^2/2} dy.
\end{align}}

\begin{thm}[Balle and Wang (2018)]
\label{th:BW}
Let $f : \mathbb{X} \rightarrow \R^d$ be a function with global $L_2$ sensitivity $\Delta > 0$. For any $\varepsilon \geq 0$ and $\delta \in [0,1]$, the Gaussian output perturbation mechanism $M(x) = \vec{x} + \vec{z}$ with $\vec{z} \sim \mathcal{N}(0,\sigma^2\mat{I}_d)$ $(\sigma > 0)$ is $(\varepsilon, \delta)$-DP if and only if
\small
\begin{align} 
\label{eq:DP-cond}
\Phi\left(\frac{\Delta}{2\sigma}-\frac{\varepsilon\sigma}{\Delta}\right) -e^{\varepsilon}\Phi\left(-\frac{\Delta}{2\sigma}-\frac{\varepsilon\sigma}{\Delta}\right) \leq \delta .
\end{align}
\end{thm}

The original proof of Theorem~\autoref{th:BW} is provided in \cite{pmlr-v80-balle18a}.
However, in \autoref{sec:proof-Th2} we provide an alternative proof, which is more concise and can be directly extended to the case of multiple neighbourhoods represented by the connected components in the nearest neighbour graph.

Theorem~\ref{th:main} (see \autoref{sec:proof-Th3} for proof) states that NADP satisfies DP conditions.
\begin{thm}[main]
\label{th:main} 
Let $\{\cX_1, \cdots, \cX_k \}$ be the connected components of the graph $G(\cX, \simeq)$ and let $\sigma_i$ $(1 \leq i \leq k)$ be non-negative real numbers such that $\sigma_i > 0$ whenever $\Delta_i = \sup_{x \simeq y, x,y\in \cX_i} \norm{\vec{x}-\vec{y}} > 0$. For any $\vec{x} \in \cX$, let $i(x)$ $(1 \leq i(x) \leq k)$ be the index such that $x \in \cX_{i(x)}$.  
Then, for any $\varepsilon \geq 0$ and $\delta \in [0,1]$, the Gaussian output perturbation mechanism $M(x) = \vec{x} + \vec{z}$ with $\vec{z} \sim \mathcal{N}(0, \sigma^2_{i(x)}\mat{I}_d)$ is $(\varepsilon, \delta)$-DP if
\small
\begin{align}\label{1}
\Phi\left(\frac{\Delta_{i(x)}}{2\sigma_{i(x)}}-\frac{\varepsilon\sigma_{i(x)}}{\Delta_{i(x)}}\right) -e^{\varepsilon}\Phi\left(-\frac{\Delta_{i(x)}}{2\sigma_{i(x)}}-\frac{\varepsilon\sigma_{i(x)}}{\Delta_{i(x)}}\right) \leq \delta
\end{align}
for any $x \in \cX$ satisfying $\Delta_{i(x)} > 0$.
\end{thm}

\begin{rem}
We have $\Delta_{i(x)} = 0$ iff the connected component $\cX_{i(x)}$ consists of only one word $x$. 
\end{rem}

\begin{algorithm}[t]
\DontPrintSemicolon
\caption{Neighbourhood-Aware Differential Privacy}\label{algo:NADP}
\KwSty{Inputs:} Connected components $\{\cX_1, \ldots, \cX_k\}$, $\epsilon \geq 0$, $\delta \in [0,1]$\;
\KwSty{Outputs:} Perturbed word embeddings $\hat{\cX} = \{\hat{\vec{x}}_1, \ldots, \hat{\vec{x}}_n\}$\;
Define $g(u) = \Phi\left( \frac{1}{2u} - \varepsilon u \right) - e^{\varepsilon} \Phi\left( -\frac{1}{2u}-\varepsilon u \right) $\;
Compute $u^{*}= \min \left\{ u \in \R_{> 0} \ \middle| \ g(u) \leq \delta \right\}$\;
Define $\hat{\cX} = \{\}$\;
\For{$i = 0, \ldots, k$}{
    Compute $\Delta_i = \sup_{x \simeq y, x,y\in \cX_i} \norm{\vec{x}-\vec{y}} $ \;
    Define $\sigma_i = u^{*} \Delta_i$\;
    \For{$x \in \cX_i$}{
        Sample $\vec{z} \sim \mathcal{N}(0, \sigma^2_i \mat{I}_d)$ \;
        $\hat{\vec{x}} = \vec{x} + \vec{z}$ \;
        $\hat{\cX} = \hat{\cX} + \{\hat{\vec{x}}\}$\;
        }
    }
Return $\hat{\cX}$
\end{algorithm}

Theorem~\ref{th:main} guarantees that the NADP mechanism described in \autoref{algo:NADP} for perturbing a set of word embeddings satisfies DP.
Specifically, we can first compute $u^*$ globally for all neighbourhoods (Line 4) as the minimiser of $g(u)$ (given by \eqref{eq:gu}) such that the DP-condition in \eqref{eq:DP-cond} is satisfied.
We can then determine the standard deviation, $\sigma_i$, corresponding to each neighbourhood, using $u^*$ and the local sensitivity, $\Delta_i$, computed from that neighbourhood.
Finally, we sample noise vectors from $\mathcal{N}(0, \sigma^2_i \mat{I}_d)$ and add to all word embeddings in each $\cX_i$.

\section{Experiments}
\label{sec:exp}

\begin{figure*}[t]
     \centering
     \begin{subfigure}[b]{0.4\textwidth}
         \centering
         \includegraphics[width=\textwidth]{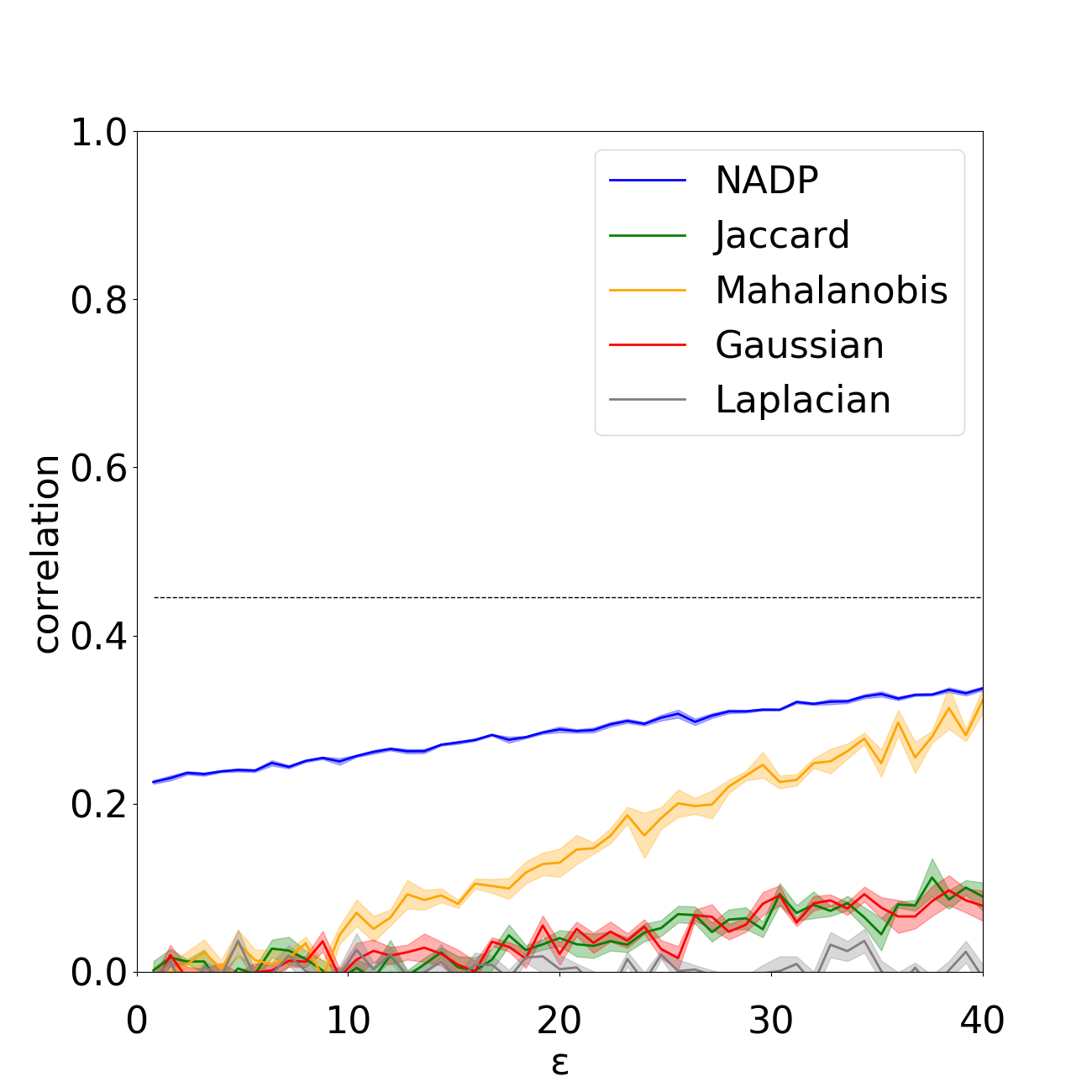}
         \caption{Word-pair similarity prediction.}
         \label{fig:word-sim}
     \end{subfigure}
     \begin{subfigure}[b]{0.4\textwidth}
         \centering
         \includegraphics[width=\textwidth]{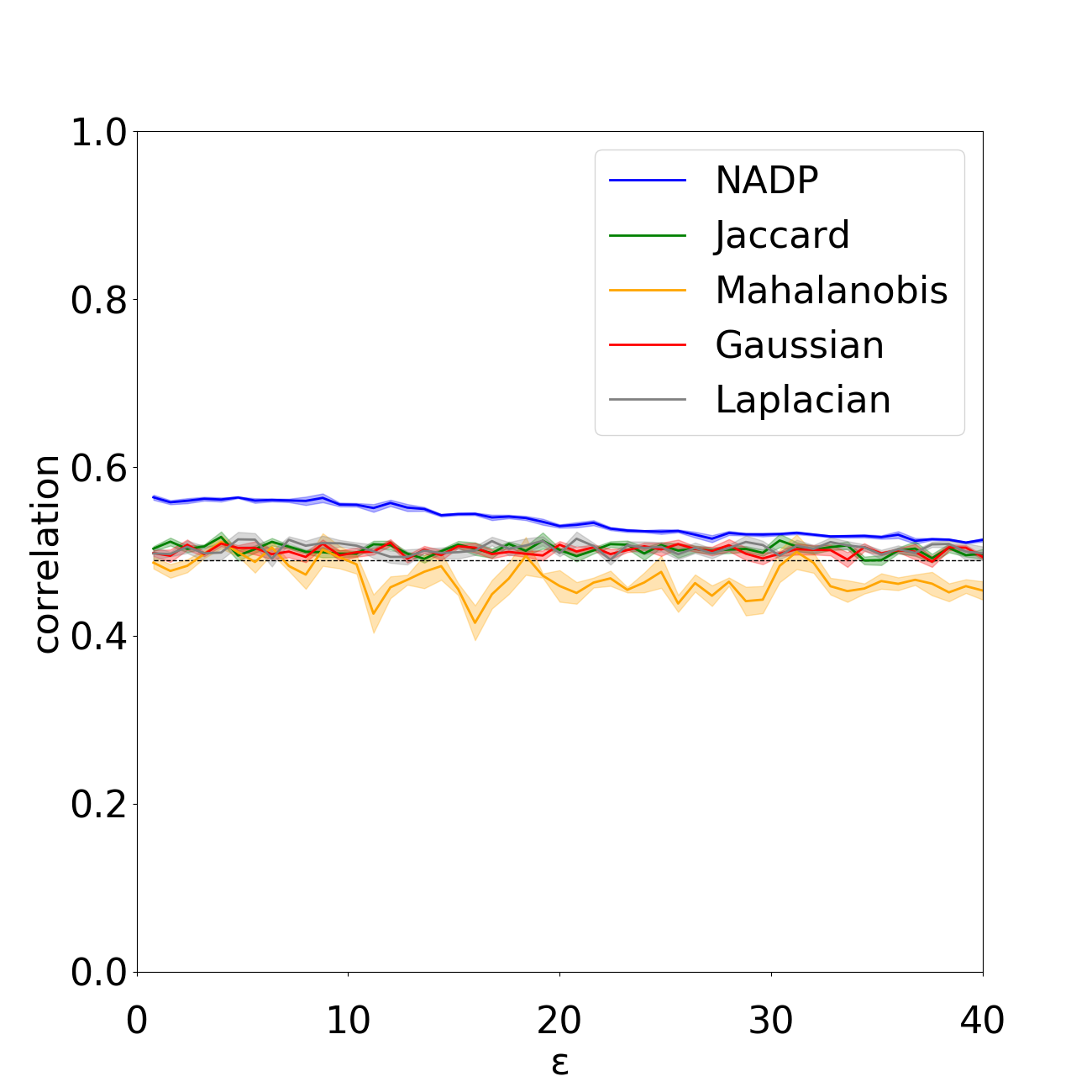}
         \caption{Semantic textual similarity measurement.}
         \label{fig:STS}
     \end{subfigure}
     \newline
     \begin{subfigure}[b]{0.4\textwidth}
         \centering
         \includegraphics[width=\textwidth]{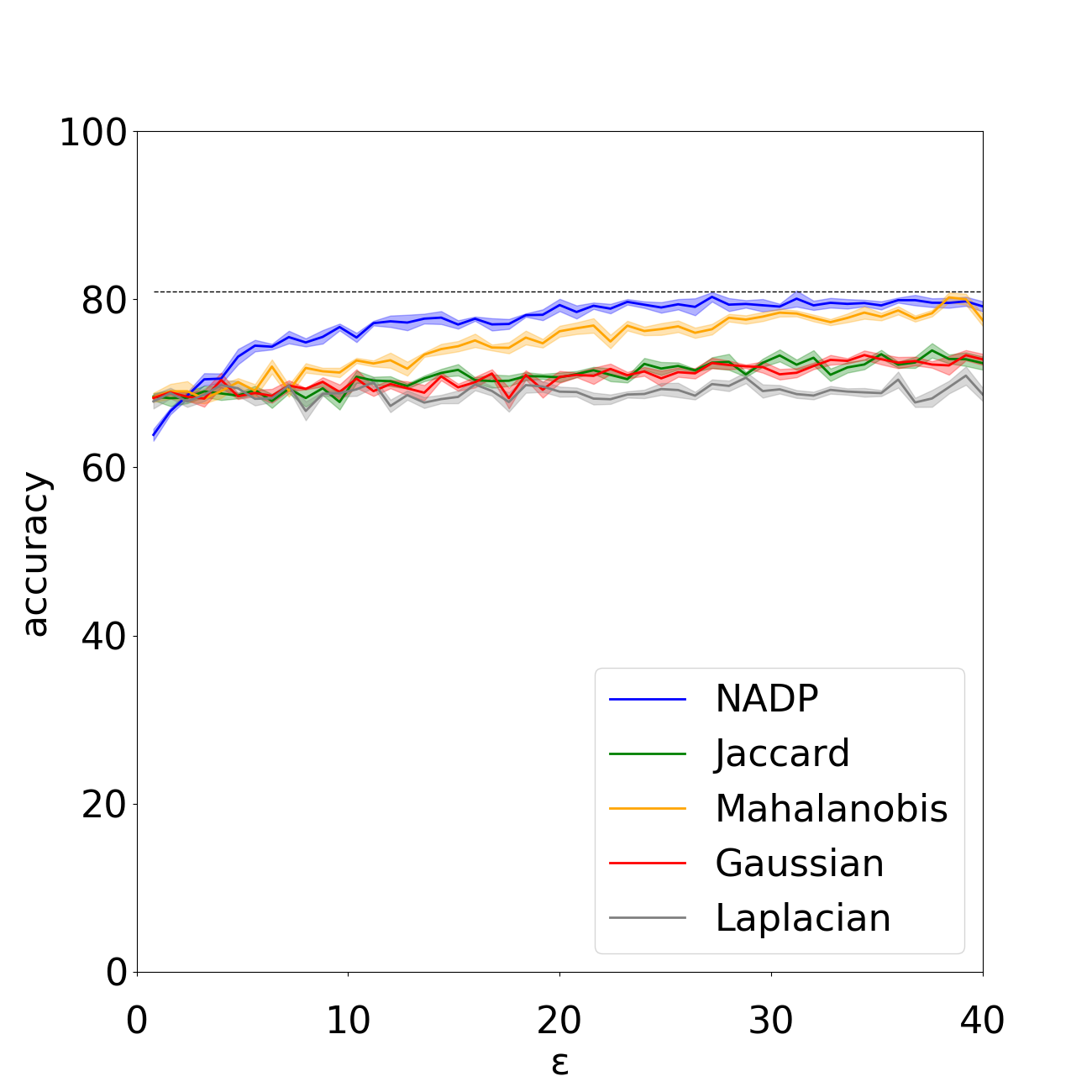}
         \caption{Text classification.}
         \label{fig:text-classification}
     \end{subfigure}
     \begin{subfigure}[b]{0.4\textwidth}
         \centering
         \includegraphics[width=\textwidth]{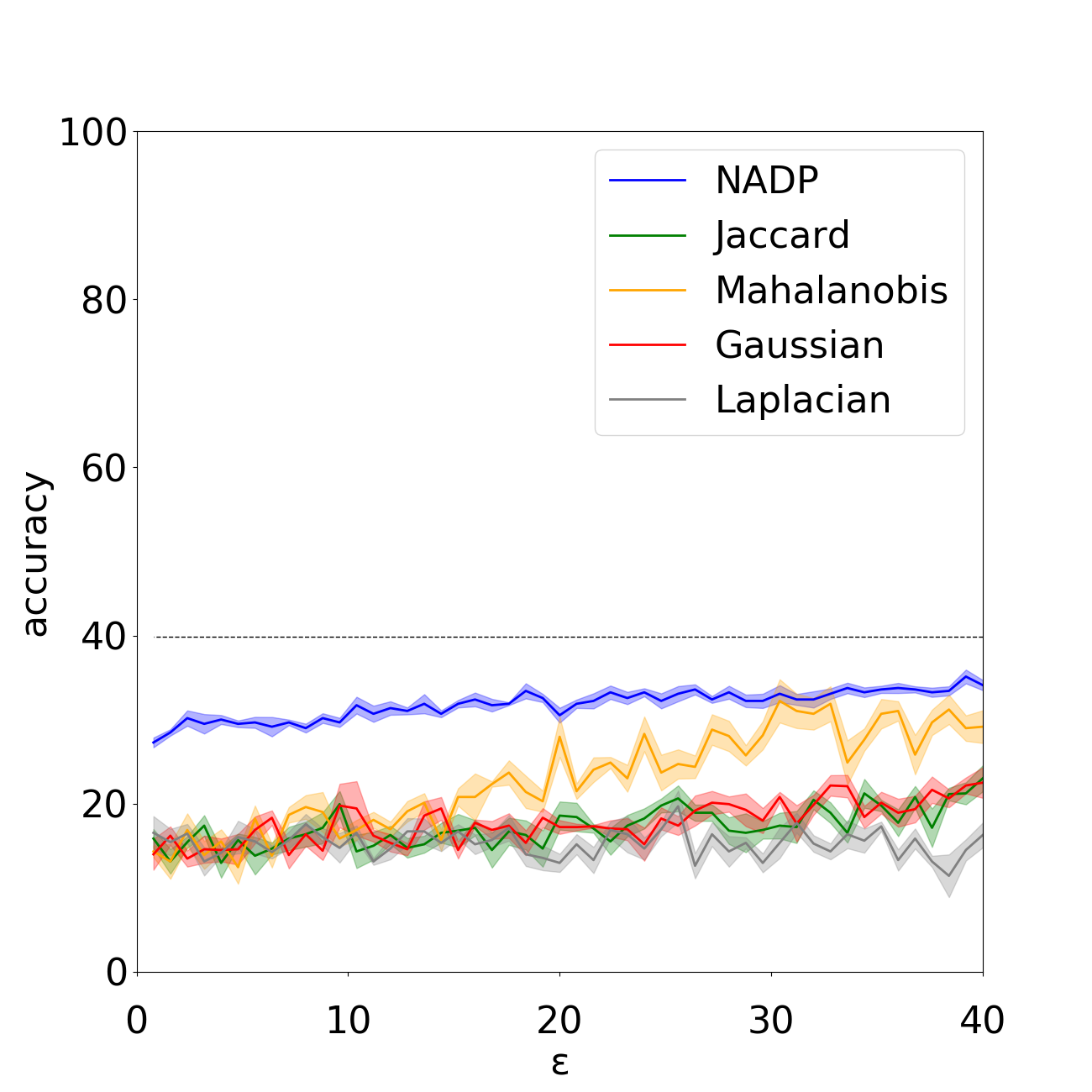}
         \caption{Odd-man-out.}
         \label{fig:odd-man-out}
     \end{subfigure}
        \caption{Performance on utility experiments (\autoref{sec:exp-utility}) shown in sub-figures (a)-(d). Accuracy and correlation (with human ratings) not decreasing with high privacy ($\epsilon$) levels (corresponding to stronger noise levels by DP mechanisms) is desirable. Performance obtained without adding any noise is shown by the horizontal dotted lines.}
        \label{fig:downstream}
\end{figure*}

We use the pretrained\footnote{We used 42B token Common Crawl trained embeddings available at \url{https://nlp.stanford.edu/projects/glove/}} 300-dimensional GloVe embeddings~\cite{Pennington:EMNLP:2014} for 2.8M words, which have also been used in much prior work~\cite{xu-etal-2020-differentially,feyisetan-kasiviswanathan-2021-private} as the static word embeddings.

We build a nearest neighbour graph using the top-1000 frequent words in the English Wikipedia, which resulted in a 73,404 vertex graph.
It takes less than 5 minutes to find all connected components of a graph containing 73,404 words used in the paper. Moreover, this is a task independent pre-processing step. Building the neighbourhood graph in a brute force manner requires 3.5 hours, while approximate nearest neighbour methods such as SCANN~\cite{scann} an be used to do the same in less than 1 minute with over 95\% recall.

In our experiments, we compare NADP against the following DP mechanisms:
\textbf{Gaussian} mechanism described in \autoref{sec:gaussian},
\textbf{Laplacian} mechanism, where noise vectors are sampled from the Laplace distribution with zero location parameter and with different values of the $\epsilon$ scale parameter,
\textbf{Mahalanobis} mechanism with the recommended parameter values by~\newcite{xu-etal-2020-differentially}  (i.e. the Mahalanobis norm $\lambda = 1$ and $\epsilon \in (0, 40]$ are used), which is the current SoTA DP mechanism for static word embeddings.

All of the above mentioned DP-mechanisms apply the same level of random noise to all word embeddings.
Therefore, to understand the importance of assigning different levels of noise to different words, we consider a baseline DP mechanism, which we call the \textbf{Jaccard} mechanism.
We define the density, $\eta(x)$, of the neighbourhood, $\cS_{k}(x)$, of a word $x$ as the average Euclidean distance between $x$ and its nearest neighbours (i.e. $\eta(x) = \frac{1}{k} \sum_{x' \in \cS_{k}(x)}\norm{\vec{x} - \vec{x'}}$).
Next, we categorise words into two density categories: dense ($\cX_{1} = \{x | x \in \cX, \eta(x) < \eta_{0}\}$) vs. sparse ($\cX_{2} = \{x | x \in \cX, \eta(x) \geq \eta_{0}\}$), based on a density threshold $\eta_0$.
Our preliminary experiments showed that splitting into more than two categories did not result in significant performance gains.
For a word $x \in \cX_{i}$, we sample a random vector $\vec{n}(x) \sim \cN(\vec{0}, \sigma_{i}\mat{I}_{d})$, for $i \in \{1,2\}$ and add to $\vec{x}$.
Jaccard is a DP mechanism (see \autoref{sec:proof-Jaccard} for the proof).
Note that the density threshold is used only by the Jaccard mechanism and is not required by NADP.
It is determined automatically such that we get approximately similar numbers of words in the dense and sparse sets.

\subsection{Utility Experiments}
\label{sec:exp-utility}

To evaluate the semantic information preserved in word embeddings, we use the following standard tasks that have been used in much prior work for this purpose~\cite{Bollegala:IJCAIa:2022,Bollegala:IJCAIb:2022,tsvetkov-EtAl:2015:EMNLP,faruqui-EtAl:2015:NAACL-HLT}: 
word similarity measurement, semantic textual similarity (STS), Text Classification, Odd-man-out~\cite{stanovsky-hopkins-2018-spot}.
Due to space limitations, we detail the tasks, datasets and evaluation metrics in \autoref{sec:tasks}.

\noindent\textbf{Results:}
\autoref{fig:downstream} shows the performance obtained on utility experiments with noise-added word embeddings for different values of the privacy parameter $\epsilon$, where we use $\tau = 0.5$.
The total set of words used in the datasets for all utility experiments is $n = 73404$.
Therefore, we set $\delta = 1/73404 \approx 0.000013623$ in all experiments reported in the paper.
We repeat each experiment five times and plot the mean and the standard error.
Recall that smaller $\epsilon$ values provide stronger DP guarantees.
From \autoref{fig:downstream}, we see that NADP reports the best performance on all four tasks among the methods compared across all $\epsilon$ values.
Among the other methods, Mahalanobis performs second best to NADP in word-pair similarity prediction, text classification and odd-man-out, but performs worst in STS.
In word-pair similarity prediction, text classification and odd-man-out tasks, we see the performance of NADP as well as the other methods increase with $\epsilon$ due to less noise being added to the word embeddings.

The performance in STS is comparatively less affected by $\epsilon$ because it is a sentence-level comparison task, which considers all perturbed word embeddings in a sentence, whereas the other three are word-level tasks.
We see that Jaccard and Gaussian mechanisms perform similarly in all tasks.
This is not surprising given that the Jaccard mechanism is drawing the noise vectors from two independent Gaussian distributions.
In particular for high $\epsilon$ values, we see that Gaussian outperforms Laplacian in word-pair similarity prediction, text classification and odd-man-out tasks.
This result implies that for making word embeddings differentially private, the $L_2$ sensitivity considered in the Gaussian mechanism is more appropriate than the $L_1$ sensitivity considered in the Laplacian mechanism.

\subsection{Privacy Experiments}
\label{sec:exp-privacy}

\begin{figure}[t!]
         \centering
         \includegraphics[width=0.45\textwidth]{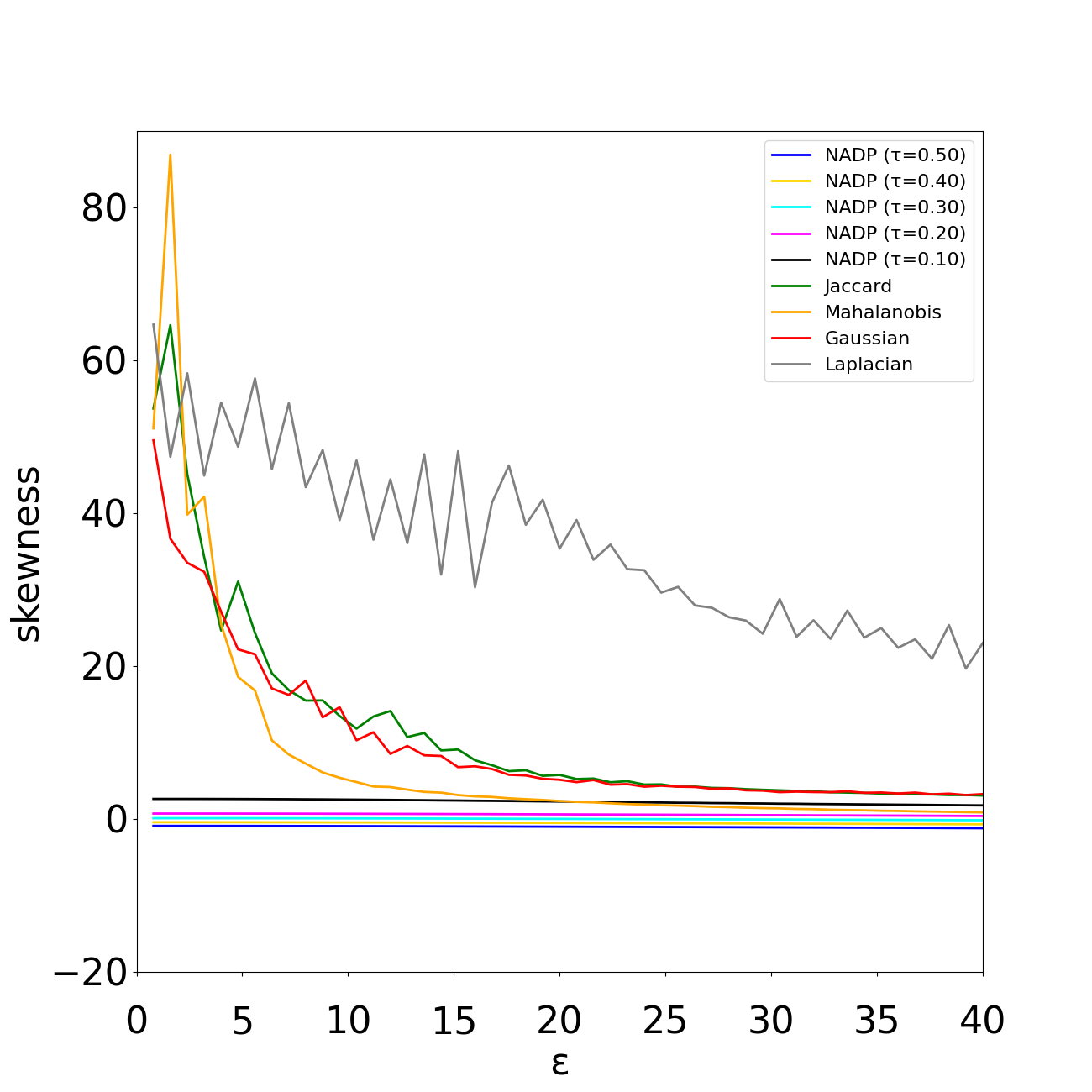}
    \caption{Skewness values for predicting words using their noise-added embeddings. Low skewness values are desirable, and indicate that the prediction probability distribution is similar to the Normal distribution and is not skewed towards a subset of the words.}
    \label{fig:skewness}
\end{figure}

\begin{table*}[t]
    \centering
    \begin{tabular}{l l l l} 
        word & no-noise & Mahalanobis & NADP \\ \toprule
        misogynist & sexist, chauvinst, bigot & insulting, sexist, racist & scholastic, filicia, shannara \\
        police & officers, cops, authorities & police, authorities, officials & posing, smiling, lying \\
        hitler & adlof, nazi, stalin & hitler, adolf, nazi & paid, ipo, raided \\
        wikileaks & assange, cia, leaked & wikileaks, iran, assange & impetuous, jashari, enraged \\
        FBI & cia, investigation, informant & fbi, history, government & asylum, cardoza, sandiego \\
    \end{tabular}
    \caption{Top 3 neighbours for words without noise addition to the embeddings (\textbf{no-noise}), with SoTA \textbf{Mahalanobis} mechanism and the proposed \textbf{NADP} mechanism. Mahalanobis mechanism sometimes discloses the original word, whereas NADP mechanism never does. }
    \label{tbl:nns}
\end{table*}

To empirically measure the level of privacy protected by a DP mechanism, we consider, $p(x|M(x))$, the probability of predicting the word $x$ using its noise-added embedding $M(x)$, as a metric of privacy provided by a DP mechanism.
However, it is difficult to accurately estimate probability densities in discrete spaces due to data sparseness.
Therefore, we approximate $p(x|M(x))$ by $\frac{|\cS_m(x) \cap \cS_m(M(x))|}{|\cS_m(x) \cup \cS_m(M(x))|}$, using the nearest neighbour sets.
It is noteworthy that this is a conservative estimate of $p(x|M(x))$ because, even if all of the nearest neighbours of $x$ and $M(x)$ fully overlap , there will still be a $1/m$ uncertainty ensuring a nonzero level of privacy.

Due to the differences in neighbourhood densities, some words are likely to be influenced more than the others by a DP mechanism.
From a DP point of view we are interested in protecting the privacy of all words in the vocabulary and not just for a subset of it.
Therefore, to empirically quantify the global effect on privacy of a DP mechanism, we compute the \emph{skewness} of the distribution of the estimated $p(x|M(x))$ values.
If most words $x_i$ has lower $p_i = p(x_i | M(x_i))$ values, the probability mass of the $p_i$ distribution will be shifted to the left of the mean, resulting in smaller skewness values (see \autoref{sec:skewness-privacy} for further explanations).
Therefore, smaller skewness values indicate that most words are protected (the probability of being discovered is smaller than the mean) under a DP mechanism.

\noindent\textbf{Results:} \autoref{fig:skewness} shows the skewness values reported by Jaccard, Mahalanobis, Gaussian, Laplacian mechanisms and the proposed NADP (for different $\tau$) mechanism for different $\epsilon$ values.
Overall, we see that NADP reports the lowest skewness values among all DP mechanism compared, indicating that it protects privacy of word embeddings well.
We see that the skewness values slightly increase with $\tau$. 
Recall that when $\tau$ increases the similarity of the neighbours connected to a target word by the symmetric neighbouring relation, $\simeq$, increases in the nearest neighbour graph $\cG$.
Therefore, when $\tau$ is high, unless when we apply stronger random noise to word embeddings, it becomes easier to discover the original word from its noise-added embedding.
However, we note that the performance of NADP is relatively unaffected by different $\tau$ values and skewness values are low for $\tau = 0.1$ setting, which we use in the utility experiments described in \autoref{sec:exp-utility}.
Although Gaussian, Jaccard and Mahalanobis mechanisms obtain comparable levels of skewness values when $\epsilon > 15$,
for $\epsilon < 5$, where stronger privacy guarantee is required, NADP is the only DP mechanism with near-zero skewness values.

\section{Investigating the Nearest Neighbours}
\label{sec:exp-NN}

To obtain qualitative insights into the levels of privacy provided by NADP, for a given word, we compare its top-3 neighbours in the original embedding space (no-noise added), when Mahalanobis and NADP mechanisms are used to add random noise.
\autoref{tbl:nns} shows the results for some randomly selected set of words.
We see that for the words such as \emph{police}, \emph{hitler}, \emph{wikileaks} and \emph{fbi}, even after applying the Mahalanobis mechanism ($\lambda = 1$), we still retrieve the original word as a nearest neighbour.
This indicates that Mahalanobis mechanism is unable to anonymise the target words in these cases.
Although not reported here due to space limitations, this problem persists even in Jaccard, Gaussian and Laplace mechanisms, which were under performing to the Mahalanobis mechanism in utility and privacy experiments.
In the case of \emph{misogynist}, Mahalanobis mechanism retrieves highly similar neighbours such as \emph{sexist}.
On the other hand, the neighbours retrieved from the word embeddings anonymised using NADP are semantically less similar to the target word, thus could be considered to be better preserving the privacy of the target word.

\section{Conclusion}
We proposed NADP to make word embeddings indistinguishable from their nearest neighbours with theoretical DP guarantees.
We compared NADP against existing DP mechanisms in multiple downstream utility experiments which showed its superior performance.
Moreover, we evaluated the level of privacy protection provided by NADP against other DP mechanisms.
We found NADP to provide stronger privacy guarantees over a broad range of $\epsilon$ values.
In our future work, we plan to extend NADP to sentence/document embeddings and evaluate for languages other than English.

\section{Ethical Considerations}

We do not annotate or release any datasets as part of this research.
However, the GloVe word embeddings that we use in our experiments are known to contain various types of unfair social biases such as gender and racial biases~\cite{Zhao:2018a,kaneko-bollegala-2019-gender,Kaneko:2021,gonen-goldberg-2019-lipstick}.
It is possible that these biases could get further amplified during the neighbourhood computation and noise-addition processes we perform in this work.
Therefore, such social biases must be properly evaluated before the noise-added word embeddings produced by our proposed method are used in real-world NLP applications that are used by users.

\section{Limitations}

Our investigations in this paper was limited to GloVe embeddings, which is one the many avaulable pre-trained static word embeddings.
There are other alternative word embeddings such as Skip-Gram with Negative Sampling (SGNS)~\cite{Milkov:2013}, PMI-based word embeddings~\cite{Arora:TACL:2016}, fastText embeddings~\cite{bojanowski-etal-2017-enriching} etc. that could be used in place of GloVe.
However, contextualised word embeddings, obtained using pre-trained Masked Language Models (MLMs) such as BERT~\cite{BERT}, RoBERTa~\cite{RoBERTa}, ALBERT~\cite{ALBERT}, etc. have reported superior performance in various downstream tasks, surpassing that by static word embeddings.
Therefore, we consider it to be a natural next step to extend our proposed method to anonymise contextualised word embeddings.
The theoretical tools that we develop in this paper should be helpful in proving DP conditions for contextualised word embeddings as well.

All the downstream datasets and word embeddings we considered in this work are limited to the English language, which is known to be a morphologically limited language.
Therefore, it is important to evaluate our proposed method on other languages using multilingual word embeddings to verify its effectiveness for the languages other than English.

\bibliography{black.bib}
\bibliographystyle{acl_natbib}

\appendix

\section*{Supplementary Materials}

\section{Downstream Tasks, Datasets and Evaluation Metrics}
\label{sec:tasks}

\begin{description}
\item[Word Similarity:] The cosine similarity between two words, computed using their word embeddings, is compared against the human similarity ratings using the Spearman correlation coefficient. 
High degree of correlation with human similarity ratings implies that the word embeddings accurately encode the word-level semantics.
We aggregate all of the word-pairs and their human similarity ratings in MEN~\cite{MEN}, SimLex~\cite{SimLex} and SimVerb~\cite{gerz-etal-2016-simverb} benchmark datasets and report the overall Spearman correlation in \autoref{fig:word-sim}.

\item[Semantic Textual Similarity (STS):]
In STS, we are provided with sentence-pairs and the human similarity ratings between the two sentences in each pair. 
Using the word embeddings, we first create an embedding for each sentence and then compute the cosine similarity between the sentence embeddings.
The correlation between the predicted sentence similarities and the human ratings is used as the evaluation metric.
We represent each sentence by the centroid of the word embeddings corresponding to the words included in that sentence.
Although this is a simple method for creating sentence embeddings from word embeddings, it is known to be a strong unsupervised baseline~\cite{Arora:ICLR:2017}, and enables us to directly attribute any differences in performance to the word embeddings -- the focus in this work.
We use the STS Benchmark dataset~\cite{cer-etal-2017-semeval}, which contains 1379 test sentence-pairs and show the official score (i.e. class-weighted geometric mean of Spearman and Pearson correlation) in \autoref{fig:STS}.

\item[Text Classification:]
We train a binary classifier to predict the sentiment (positive vs. negative) of a short review text.
Similar to the STS task, we represent a review using the centroid of the word embeddings of the words included in that review.
We train a binary logistic regression model to predict sentiment in a review and in \autoref{fig:text-classification} report the averaged classification accuracy on the balanced test sets in three standard datasets: Movie reviews dataset~\cite{Pang:ACL:2005}, customer reviews dataset~\cite{Hu:KDD:2004} and opinion polarity dataset~\cite{MPQA}.

\item[Odd-man-out:]
Stanovsky and Hopkins~\cite{stanovsky-hopkins-2018-spot} proposed the \emph{odd-man-out} task, where given a set of five or more words, a system is required to choose the one which does not belong with the others. 
They annotated a dataset containing 843 sets via crowd sourcing.
Pretrained word embeddings can be used to identify the odd-man in a set by repeatedly excluding one word at a time and measuring the average cosine similarity between all remaining pairs of words.
Finally, the word when excluded resulting in the highest pairwise similarity is chosen as the odd-man.
Unlike previously described tasks, odd-man-out can be carried out in an unsupervised manner, at word-level, and has higher human agreement between the annotators because it does not require numerical ratings.
The percentage of correctly solved sets is shown in \autoref{fig:odd-man-out}.
\end{description}

\section{Skewness and Privacy}
\label{sec:skewness-privacy}

Skewness is a measure of the asymmetry of $p(x|M(x))$ about its mean and can be positive, negative or zero depending on whether $p(x|M(x))$ has respectively a longer left tail, right tail, or perfectly symmetric around the mean (e.g. as in the case of the standard Normal distribution)~\cite{Joanes:1998}.
Specifically, if we denote the probability of predicting $i$-th word $x_i$ by $p_i = p(x_i | M(x_i))$, the skewness of the distribution of $p_i$ over $n$ words is given by
$\frac{n}{(n-1)(n-2)}\sum_{i=1}^{n}{\left( \frac{p_i - \bar{p}_i}{s} \right)}^3$, where $\bar{p}_i$ and $s$ are respectively the mean and standard deviation of $\{p_i\}_{i=1}^{n}$.

\begin{figure*}[t!]
    \centering
    \includegraphics[width=\textwidth]{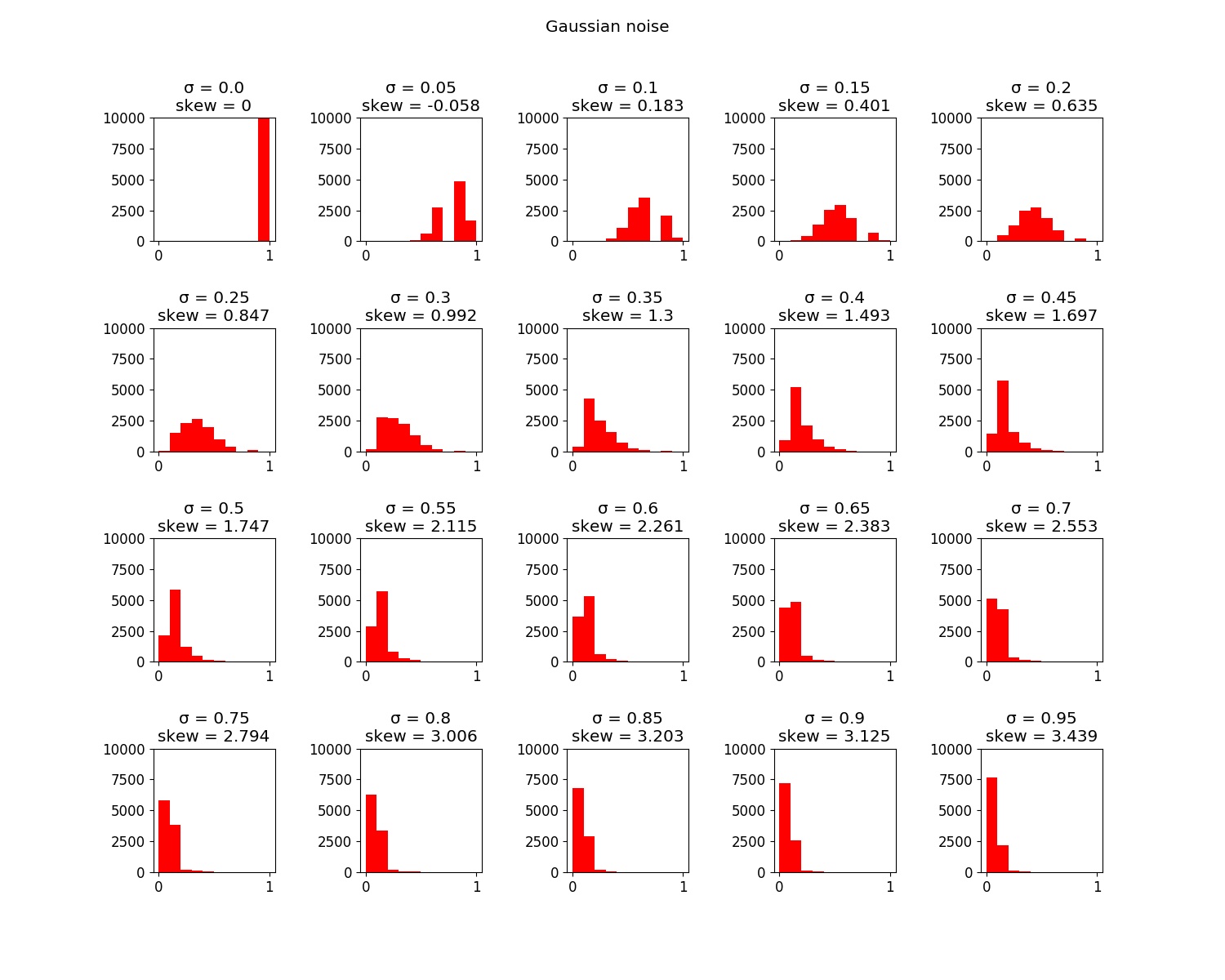}
    \caption{Histogram of $p_i$ values when zero-mean and $\sigma$ standard deviation Gaussian noise is added to the word embeddings. Skewness values (skew) are shown in each histogram alongside with the $\sigma$.}
    \label{fig:skeweness-privacy}
\end{figure*}

We study the relationship between the level of privacy protected by the noise added using a particular DP mechanism, $M$ and the skewness of the distribution of $p(x|M(x))$ for the words $w$ in a vocabulary $\cX$.
For this purpose, we use the Gaussian mechanism described in \autoref{sec:gaussian} in the paper where we sample noise vectors $\vec{z} \in \R^d$ from the $d$-dimensional spherical Gaussian $\cN(0, \sigma\mat{I}_d)$ with zero-mean and standard deviation $\sigma$, and add this noise to the word embedding, $\vec{x} \in \R^d$, representing the word $x$. Specifically, $M(x) = \vec{x} + \vec{z}$.
Next, we gradually increase $\sigma \in [0,1]$ in step size of $0.05$ and compute the histograms of $p(x|M(x))$ values for the words in $\cX$.
The histograms and their skewness values are shown in \autoref{fig:skeweness-privacy}.

From \autoref{fig:skeweness-privacy}, we see that when no-noise is being added (i.e. $\sigma = 0$), the histogram peaks at $1$, indicating that all words can be trivially discovered from their word embeddings because the closest neighbour of any target word in the embedding space will be itself.
Because the distribution is symmetric around this peak, we have a zero skewness.
Overall, we see that when we add increasingly high noise, the histograms start shifting towards to the left because less words will be perfectly discovered from the noise added embeddings.
Moreover, we see that more probability mass is distributed towards the right side of the mode (peak), resulting in a longer right tail.
Consequently, we see skewness values also continuously increase (except at $\sigma = 0.05$, where the distribution has split into two parts) with $\sigma$.
This trend stems from the definition of skewness and is independent of the DP mechanism used to generate noise.
This result shows that when there are many words with smaller  $p(x|M(x))$ values (i.e. distribution has a longer left tail), the skewness values will be smaller, indicating that the privacy is preserved for many words in $\cX$.

\section{Proofs of Theorems}

\subsection{Proof of Theorem~2}
\label{sec:proof-Th2}

\begin{proof}
For any $\varepsilon \geq 0$ and $\delta \in (0,1)$, put 
{\small
\begin{align}
\label{eq:gu}
g(u) = 
\Phi\left( \frac{1}{2u} - \varepsilon u \right) - e^{\varepsilon} \Phi\left( -\frac{1}{2u}-\varepsilon u \right) 
\end{align}}
and $u^{*}= \min \left\{ u \in \R_{> 0} \ \middle| \ g(u) \leq \delta \right\}$. \\[2mm]
We will prove the following three statements:\\
(a) $g(u)$ is strictly decreasing on $(0, \infty)$ with $\lim_{u \to +0} g(u) = 1$, $\lim_{u \to \infty} g(u)= 0$. \\
(b) The equation $g(u) = \delta$ has the unique solution $u^{*}$. \\
(c) Put $\sigma = u^{*}\Delta$. Then, 
the Gaussian output perturbation mechanism $M(x) = \vec{x} + \vec{z}$ with $\vec{z} \sim \mathcal{N}(0, \sigma^2\mat{I}_d)$ is $(\varepsilon, \delta)$-DP.\\
To prove (a) Put
{\small
\begin{align*}
 h(u) &= \sqrt{2\pi}g(u) \\
 &= \sqrt{2\pi}\left\{\Phi\left(\frac{1}{2u}-\varepsilon u \right) -e^{\varepsilon}\Phi\left(-\frac{1}{2u}-\varepsilon u\right)\right\}.
\end{align*}}
Then,
{\small
\begin{align*}
h^{\prime}(u) &= \exp\left( -\frac{1}{2}\left( \frac{1}{2u} - \varepsilon u \right)^2 \right) \cdot \left( -\frac{1}{2u^2}-\varepsilon \right) \\
&-\exp(\varepsilon) \cdot \exp\left( -\frac{1}{2}\left( -\frac{1}{2u} - \varepsilon u \right)^2 \right) \cdot \left( \frac{1}{2u^2}-\varepsilon \right) \\
&= \exp\left( -\frac{1}{2}\left( \frac{1}{2u} - \varepsilon u \right)^2 \right) \cdot \left( -\frac{1}{2u^2}-\varepsilon \right) \\
&
-\exp\left( -\frac{1}{2}\left( \frac{1}{2u} - \varepsilon u \right)^2 \right) \cdot \left( \frac{1}{2u^2}-\varepsilon \right) \\
&=-\frac{1}{u^2} \cdot \exp\left( -\frac{1}{2}\left( \frac{1}{2u} - \varepsilon u \right)^2 \right) < 0
\end{align*}}
for any $u > 0$. Therefore, $g(u) = (1/\sqrt{2\pi})h(u)$ is strictly decreasing on $(0, \infty)$. The latter half of the statement is clear from the definition of $g(u)$. \\
Next, to prove (b) observe that since $g(u)$ is a continuous function on $(0, \infty)$ satisfying $\lim_{u \to +0} g(u) = 1$ and $\lim_{u \to \infty} g(u)= 0$, $g(u) = \delta$ has a solution $u^{\prime}$ for any $\delta \in (0,1)$, which must be unique and satisfy $u^{\prime}=u^{*}$ because of the monotonicity of $g(u)$. \\
Finally, to prove (c) let $\sigma = u^{*}\Delta$. We then have
{\small
\begin{align*}
&\Phi\left(\frac{\Delta}{2\sigma}-\frac{\varepsilon\sigma}{\Delta}\right) -e^{\varepsilon}\Phi\left(-\frac{\Delta}{2\sigma}-\frac{\varepsilon\sigma}{\Delta}\right) \\
&= \Phi\left( \frac{1}{2u^*} - \varepsilon u^* \right) - e^{\varepsilon} \Phi\left( -\frac{1}{2u^*}-\varepsilon u^* \right) = g(u^{*}) \leq \delta, 
\end{align*}}
which implies the mechanism $M(x) = \vec{x} + \vec{z}$ with $ \vec{z} \sim \mathcal{N}(0, \sigma^2\mat{I}_d)$ is $(\varepsilon, \delta)$-DP as stated in Theorem~\ref{th:BW}.
\end{proof}

\subsection{Proof of Theorem~3}
\label{sec:proof-Th3}

\begin{proof}
For any $i$ $(1 \leq i \leq k)$, let $\simeq_i$ be the symmetric neighbouring relation obtained by restricting the relation $\simeq$ on $\cX_i$. Then, $\Delta_i$ equals to the global $L_2$ sensitivity of $(\cX_i, \simeq_i)$ and $\Delta_{i(x)} = \Delta_i$, $\sigma_{i(x)} = \sigma_i$ for any $x \in \cX_i$. Hence, if $\Delta_{i} > 0$, the mechanism $M_i$ obtained by restricting $M$ on $(\cX_i, \simeq_i)$ is $(\varepsilon, \delta)$-DP if and only if
{\small
\begin{align*}
\Phi\left(\frac{\Delta_{i(x)}}{2\sigma_{i(x)}}-\frac{\varepsilon\sigma_{i(x)}}{\Delta_{i(x)}}\right) -e^{\varepsilon}\Phi\left(-\frac{\Delta_{i(x)}}{2\sigma_{i(x)}}-\frac{\varepsilon\sigma_{i(x)}}{\Delta_{i(x)}}\right) \leq \delta 
\end{align*}}
for any $x \in \cX_i$ by Theorem~\ref{th:BW}. 

Let $x, x^{\prime} \in \cX$ be words such that $x \simeq x^{\prime}$ and $i$ $(=i(x)=i(x^{\prime}))$ be the index such that $x ,x^{\prime} \in \cX_{i}$, that is, $x \simeq_{i} x^{\prime}$. Now, suppose $\Delta_i > 0$ and (\ref{1}) holds for any $x \in \cX_i$. Then, the mechanism $M_i$ is $(\varepsilon, \delta)$-DP and hence, we have
\begin{align*}
\Pro[M(x) \in E] &= \Pro[M_{i}(x) \in E] \\
&\leq e^{\varepsilon}\Pro[M_{i}(x^{\prime}) \in E] + \delta \\
&= e^{\varepsilon}\Pro[M(x^{\prime}) \in E] + \delta
\end{align*}
for any measurable set $E \subset \R$. Next, suppose $\Delta_{i(x)} = \Delta_{i(x^{\prime})} = 0$. Then, we have $x = x^{\prime}$ and hence
\begin{align*}
\Pro[M(x) \in E] \leq e^{\varepsilon}\Pro[M(x^{\prime}) \in E] + \delta
\end{align*}
for any measurable set $E \subset \R$, which implies the mechanism $M$ is $(\varepsilon, \delta)$-DP if the condition (\ref{1}) holds for any $x \in \cX$ satisfying $\Delta_{i(x)} > 0$. 
\end{proof}

\begin{cor}
For any $\varepsilon \geq 0$ and $\delta \in (0,1)$, put 
{\small
\begin{align*}
g(u) = 
\Phi\left( \frac{1}{2u} - \varepsilon u \right) - e^{\varepsilon} \Phi\left( -\frac{1}{2u}-\varepsilon u \right) 
\end{align*}}
and $u^{*}= \min \left\{ u \in \R_{> 0} \ \middle| \ g(u) \leq \delta \right\}$. Also, put $\sigma_{i(x)} = u^{*}\Delta_{i(x)}$ for any $x \in \cX$.
Then, the Gaussian output perturbation mechanism $M(x) = \vec{x} +\vec{z}$ with $\vec{z} \sim \mathcal{N}(0, {\sigma}^2_{i(x)}\mat{I}_d)$ is $(\varepsilon, \delta)$-DP.
\end{cor}

\begin{proof}
Suppose $\Delta_{i(x)}>0$. Then, by definition, we have
{\small
\begin{align*}
&\Phi\left(\frac{\Delta_{i(x)}}{2\sigma_{i(x)}}-\frac{\varepsilon\sigma_{i(x)}}{\Delta_{i(x)}}\right) -e^{\varepsilon}\Phi\left(-\frac{\Delta_{i(x)}}{2\sigma_{i(x)}}-\frac{\varepsilon\sigma_{i(x)}}{\Delta_{i(x)}}\right) \\
&= \displaystyle \Phi\left( \frac{1}{2u^*} - \varepsilon u^* \right) - e^{\varepsilon} \Phi\left( -\frac{1}{2u^*}-\varepsilon u^* \right) 
= g(u^{*}) \leq \delta,
\end{align*}}
which implies the mechanism $M(x) = \vec{x} +\vec{z}$ with $\vec{z} \sim \mathcal{N}(0, {\sigma}^2_{i(x)}\mat{I}_d)$ is $(\varepsilon, \delta)$-DP as claimed in Theorem~\ref{th:main}.
\end{proof}

\subsection{Jaccard Mechanism is DP}
\label{sec:proof-Jaccard}

Theorem~\ref{th:jaccard} claims that the above-mentioned Jaccard mechanism is $(\epsilon,\delta)$-DP.

\begin{thm}[Jaccard mechanism is DP]
\label{th:jaccard}
Jaccard mechanism with $\sigma_{i} = \Delta \alpha_{i}\sqrt{2 \log(1.25/\delta)}/\epsilon$ is $(\epsilon,\delta)$-DP. Here, $\alpha_{i}$ is a constant that depends only on the density category of a word and $\Delta$ is the global sensitivity over the vocabulary.
\end{thm}
\begin{proof}
Note that under the Jaccard mechanism, noise vectors, $n(x)$, are sampled from either one of the two Gaussians $\cN(\vec{0}, \sigma_{1}\mat{I}_{d})$ or $\cN(\vec{0}, \sigma_{2}\mat{I}_{d})$ depending on respectively whether $x \in \cX_{1}$ or $x \in \cX_{2}$.
Moreover, because $\alpha_{i}$ depends only on $\cX_{i}$,  from $\sigma_{i} = \Delta \alpha_{i}\sqrt{2 \log(1.25/\delta)}/\epsilon$ and from Theorem~\ref{th:gauss} we see that each of these underlying Gaussian mechanisms are $(\epsilon,\delta)$-DP.
Because $\cX_{1} \cap \cX_{2} = \emptyset$ by their definitions, it follows from the compositionality property of DP that the overall Jaccard process is also $(0, \delta)$-DP.
This proof can be easily extended to more than two density categories by mathematical induction.
\end{proof}
In our experiments, we use $\eta_0 = 6.0$ such that approximately equal numbers of words in $\cX$ belong to each category, corresponding to $\alpha_1 = 1.835$ and $\alpha_2 = 1.276$ for $m = 10$.
Global sensitivity $\Delta$ is computed as the average Euclidean distance between a word and its furthermost neighbour.

The ability to guarantee the mean overlap between neighbourhoods before and after the noise addition is important from the point-of-view of NLP tasks that depend on the neighbourhood information such as semantic similarity measurement, bag-of-words representations-based information retrieval and word/text classification tasks, etc.
Unlike in the Gaussian mechanism, in the Jaccard mechanism we have a direct relationship between the level of noise and the performance obtained using the anonymised embeddings in the downstream tasks.
Moreover, the Jaccard mechanism allows us to set different noise levels to sparse vs. dense regions in the embedding space, which is not possible with other DP mechanisms.

\end{document}